\documentclass[11pt]{article} 
\pdfoutput=1

\usepackage{etoolbox}
\newtoggle{colt}
\togglefalse{colt}
\usepackage[utf8]{inputenc}
\usepackage{mathrsfs}
\usepackage{amsmath, amsfonts, amssymb, graphicx}
\iftoggle{colt}
{
	\usepackage{times}
}
{
	\usepackage{amsthm,pdfpages}
	\usepackage{hyperref}
	\usepackage{cleveref}
	\hypersetup{colorlinks,citecolor=blue}
	\usepackage{fancyhdr, lastpage}
	\usepackage[vmargin=1.00in,hmargin=1.00in,centering,letterpaper]{geometry}
	\setlength{\headsep}{.10in}
	\setlength{\headheight}{15pt}
	\cfoot{}
	\lfoot{}
	\rfoot{\sc Page\ \thepage\ of\ \protect\pageref{LastPage}}

}

\usepackage{mathtools,verbatim}
\usepackage{enumitem}
\usepackage{natbib}
\usepackage{tikz}
\usetikzlibrary{bayesnet}
\usepackage[noend]{algpseudocode}
\usepackage[linesnumbered,ruled]{algorithm2e}

\newcommand{\norm}[1]{\left\|#1\right\|}

\iftoggle{colt}
{}
{
	
}
\newcommand{\pars}[1]{\left(#1\right)}
\newcommand{\bracks}[1]{\left[#1\right]}

\newcommand{\calN}{\mathcal{N}}

\newcommand{\matU}{\mathbf{U}}

\newcommand{\calE}{\mathcal{E}}

\iftoggle{colt}
{
		\newtheorem{claim}[theorem]{Claim}

}
{
	{
	 \theoremstyle{plain}
	      
	}
	\newtheorem{nono-theorem}{Theorem}[]
	\theoremstyle{plain}
	\newtheorem{theorem}{Theorem}[section]
	\newtheorem{claim}[theorem]{Claim}
	\newtheorem{lemma}[theorem]{Lemma}
	\newtheorem{corollary}[theorem]{Corollary}

	\newtheorem{remark}{Remark}

	\newtheorem{proposition}[theorem]{Proposition}
	\theoremstyle{definition}
	\newtheorem{definition}{Definition}[section]

}

\renewcommand{\Pr}{\mathbb{P}}
\newcommand{\Exp}{\mathbb{E}}

\newcommand{\nnz}{\#\textrm{nnz}}

\newcommand{\unifsim}{\overset{\mathrm{unif}}{\sim}}
\newcommand{\iidsim}{\overset{\mathrm{i.i.d.}}{\sim}}

\newcommand{\op}{\mathrm{op}}

\newcommand{\matX}{\mathbf{X}}

\newcommand{\maty}{\mathbf{y}}

\newcommand{\N}{\mathbb{Z}}

\newcommand{\R}{\mathbb{R}}
\newcommand{\I}{\mathbb{I}}

\newcommand{\PD}{\mathbb{S}_{++}^d}
\newcommand{\PDof}[1]{\mathbb{S}_{++}^{#1}}

\newcommand{\sphered}{\mathcal{S}^{d-1}}

\newcommand{\matW}{\mathbf{W}}
\newcommand{\matM}{\mathbf{M}}
\newcommand{\gap}{\mathrm{gap}}
\newcommand{\xhat}{\widehat{x}}
\newcommand{\vhat}{\widehat{v}}
\newcommand{\lamhat}{\widehat{\lambda}}
\renewcommand{\N}{\mathbb{N}}

\newcommand{\poly}{\mathrm{poly}}
\newcommand{\calD}{\mathcal{D}}

\newcommand{\cond}{\mathrm{cond}}

\newcommand{\Mclass}{\mathscr{M}_d}

\newcommand{\Sym}{\mathbb{S}}
\newcommand{\Sympl}{\mathbb{S}_{++}}

\newcommand{\BigOhSt}[1]{\BigOm^*\left({#1}\right)}
\DeclareMathOperator{\BigOm}{\mathcal{O}}
\newcommand{\BigOhPar}[2]{\BigOm_{#1}\left({#2}\right)}

\newcommand{\BigOh}[1]{\BigOm\left({#1}\right)}
\DeclareMathOperator{\BigOmtil}{\widetilde{\mathcal{O}}}
\newcommand{\BigOhTil}[1]{\BigOmtil\left({#1}\right)}
\DeclareMathOperator{\BigTm}{\Theta}
\newcommand{\BigTheta}[1]{\BigTm\left({#1}\right)}
\DeclareMathOperator{\BigWm}{\Omega}

\newcommand{\BigOmega}[1]{\BigWm\left({#1}\right)}

\newcommand{\BigOmegaTil}[1]{\widetilde{\BigWm}\left({#1}\right)}

\newcommand\smallO{
  \mathchoice
    {{\scriptstyle\mathcal{O}}}
    {{\scriptstyle\mathcal{O}}}
    {{\scriptscriptstyle\mathcal{O}}}
    {\scalebox{.7}{$\scriptscriptstyle\mathcal{O}$}}
  }

\newcommand{\LittleOm}{\smallO}
\newcommand{\LittleOh}[1]{\LittleOm\left({#1}\right)}

\DeclareMathOperator*{\argmax}{arg\,max}

\newcommand{\Alg}{\mathsf{Alg}}

\newcommand{\vi}{\matv^{(i)}}

\title{The gradient complexity of linear regression}

\author{{Mark Braverman\thanks{Princeton University. \url{mbraverm@cs.princeton.edu}}}\and { Elad Hazan\thanks{Princeton University and Google AI Princeton. \url{ehazan@cs.princeton.edu}}} \and Max Simchowitz\thanks{UC Berkeley. \url{msimchow@berkeley.edu}} \and Blake Woodworth\thanks{Toyota Technological Institute at Chicago. Work done while visiting Google AI Princeton. \url{blake@ttic.edu}}}

\begin{document}
\begin{titlepage}
\maketitle{}
\begin{abstract}

We investigate the computational complexity of several basic linear algebra primitives, including largest eigenvector computation and linear regression, in the computational model that allows access to the data via a matrix-vector product oracle. We show that for polynomial accuracy, $\Theta(d)$ calls to the oracle are necessary and sufficient even for a randomized algorithm. 

Our lower bound is based on a reduction to estimating the least eigenvalue of a random Wishart matrix. This simple distribution enables a concise proof, leveraging a few key properties of the random Wishart ensemble.

\end{abstract}
\end{titlepage}

\section{Introduction}

Solving linear systems and computing eigenvectors are fundamental problems in numerical linear algebra, and have widespread applications in numerous scientific, mathematical, and computational fields. Due to their simplicity, parallelizability, and limited computational overhead, first-order methods based on iterative gradient updates have become increasingly popular for solving these problems. Moreover, in many settings, the complexity of these methods is currently well understood: tight upper and lower bounds are known for gradient methods, accelerated gradient methods and related algorithms. 

\paragraph{First order methods for regression and eigenvector computation. }
As an example, consider the problem of computing the largest eigenvector for a given matrix $M \in \R^{d\times d}$. The power method finds an $\epsilon$-approximate solution in $\BigOh{\frac{\log d }{\epsilon}}$ iterations, each involving a matrix-vector product that can be computed in time proportional to the number of non-zeros in the matrix. A variant of the Lanczos algorithm improves this complexity to $\BigOh{\frac{\log d}{\sqrt{\epsilon}}}$ \citep{kuczynski1992estimating,musco2015randomized}.  Alternatively, if the matrix has an inverse-eigengap $\frac{\lambda_1(M)}{\lambda_1(M) - \lambda_2(M)}$ bounded by $\kappa$, the above running times can be improved to $\BigOh{\kappa \log \frac{d}{\epsilon}}$ and $\BigOh{\sqrt{\kappa}\log \frac{d}{\epsilon}}$.  In a low accuracy regime, where $\epsilon \gg 1/d$, these upper bounds attained by Lanczos are known to be information-theoretically tight in the number of matrix-vector products required to compute a solution to the given precision \citep{simchowitz2018tight}. The optimal complexities are nearly identical for solving linear systems, except that these do not incur a $\log d$ dependence on the ambient dimension $d$ \citep{simchowitz2018randomized}. More generally, these upper and lower bounds extend to convex optimization with first order methods more broadly \citep{nemirovskii1983problem}.

\paragraph{The blessing of data sparsity. } One major advantage of first order methods is that they benefit from \emph{sparsity}. Each iteration of first order methods computes $\BigOh{1}$ matrix-vector multiplies, and if the matrix in question has $\nnz$ non-zero entries, then these multiplications can be performed in $\BigOh{\nnz}$-time. This yields runtimes which scale with the sparsity of the problem instance, rather than the ambient dimension (which can be quadratically worse). 


\begin{table}[htp]
\begin{center}
\begin{tabular}{|c|c|c|}
\hline
Regime of Dominance &  Running time & Method \\
\hline
\hline
$ \epsilon \ge 1/\text{poly}(d)$ & $ \frac{1}{\sqrt{\epsilon}} \times (\nnz )$ & Lanczos, CG, AGD \\
\hline
$ \kappa < \frac{1}{\epsilon}$ & $ \sqrt{\kappa} \log  \frac{1}{{\epsilon}} \times (\nnz) $ & Lanczos, CG, AGD \\
\hline
$ \kappa ,\epsilon^{-1} \ge d^2$ & $ d \times (\nnz)$ & Lanczos \& CG (not AGD) \\
\hline
$ \nnz = d^{2}$ (ties with above) & $  d^{3}$ & matrix inversion (naive) \\
\hline
$ \nnz \ge  d^{\omega - 1}$ & $  d^{\omega} + d^2  \log \frac{1}{\epsilon}  $ & matrix inversion (state of art) \\
\hline
\end{tabular}
\end{center}
\caption{Methods for computing the largest eigenvector of $A \in R^{d \times d}$, and solving  linear systems with $d$-data points, equivalent to computing $A^{-1}b$ for $b \in R^d$. Here $\nnz$ denotes the number of nonzero entries of $A$, $\kappa$ refers to an upper bound the condition number (in least squares) or eigengap (in PCA). For eigenvalue problems, above runtimes suppress $\log d$-dependence. Lanczos refers to the block Lanczos methods \citep{musco2015randomized}, CG to the conjugate gradient methods ( e.g. \cite{trefethen1997numerical}), AGD to accelerated gradient descent. We use 'naive' matrix elimination refers to approaches such as those based on Gaussian elimination; `state of art' matrix inversion denotes the theoretically state-of-art approach due to \cite{williams2012multiplying}, which enjoys exponent $\omega \approx 2.3727$. Logarithmic factors associated with iteration complexity are included; logarithmic factors associated with numerical precision are suppressed.} 
\label{default}
\end{table}%




\paragraph{The high accuracy regime.} What is the computational complexity of obtaining high, inverse polynomial  $\epsilon = \frac{1}{\poly(d)}$ precision using a randomized algorithm, without a bound on the condition number or eigengap of the matrix? 
This is precise the gap in the literature that our results address. 

In this regime, our understanding of even these most basic problems is poor by comparison. The best known algorithms for $\epsilon$-approximation in this regime scale as $\BigOh{d^\omega + d^2 \log \frac{1}{\epsilon}}$, where $\omega$ is the matrix inversion constant, currently around $2.37$. These methods proceed by attempting to invert the matrix in question. Since the inverse of a sparse matrix is itself not necessarily sparse, these methods do not take advantage of data sparsity.

It is thus natural to ask if there is a randomized algorithm based on gradient-like queries that can exploit data sparsity, even for the simplest of linear algebra problems sketched above. We note that such faster algorithms would not necessarily require an inverse of the entire matrix, and therefore would not imply faster matrix inversion. 
\paragraph{Our results.} Our main result shows that first-order methods based on gradient-like queries cannot significantly surpass the performance of the existing conjugate gradient or Lanczos methods, even for the simplest of linear algebra problems sketched above, and crucially, \emph{even in the high-accuracy regime.}

In a computation model where each iteration corresponds to one query of a matrix-vector product, we show that $\Omega(d)$ matrix-vector product oracle queries are necessary to obtain a $1/d^2$-accurate approximation to the largest eigenvector. This is tight, as $d$ such queries are sufficient for the Lanczos method to obtain an exact solution (up to machine precision). Similarly, we show a lower bound of $\BigOmegaTil{d}$ queries for solving linear systems, which nearly matches the $d$-query upper bound of the conjugate gradient method. 

Moreover, for instances with $\nnz(A) = \Theta(s^2)$ nonzero entries. we show a lower bound of $\Omega(s)$ queries necessary for high-precision eigenvalue approximation, and $\BigOmegaTil{s}$ for solving linear systems. This suggests an overall computational complexity of $\Omega(s^3)$ for first order methods. This in turn demonstrates that algebraic methods based on matrix inversion asymptotically outperform optimization-based approaches in the regime $s \ge d^{\omega/3}$. 

Finally, our lower bounds are constructed so that the instance sparsity $s$ encodes the eigengap (resp. condition number)  parameters for eigenvector approximation (resp. least squares). In turn, these parameters can in turn be used to encode target accuracy $\epsilon$ in the low-accuracy regime. When translated in terms of these parameters, our guarantees are near-optimal up to logarithmic factors in terms of both eigengap/condition number and accuracy.  In contrast to much existing work, our lower bounds are \emph{information theoretic}, and apply to randomized algorithms, even those that do not satisfy Krylov restrictions. To our knowledge, this is the first work that provides lower bounds which apply to general randomized algorithms, and attain optimal dimension-dependence in the high accuracy regime when $\epsilon \ll 1/\poly(d)$. For a thorough discussion of the prior art, see our discussion of related work below.

\paragraph{Randomized algorithms}
Our work establishes lower bounds for \emph{randomized} algorithms. These are more interesting than lower bounds for deterministic algorithms for several reasons. Of course, the former are stronger and more widely applicable than the latter. More importantly, there are problems for which randomized algorithms can outperform deterministic algorithms enormously, for instance, the only polynomial time algorithms for volume computation are randomized \citep{lovasz2006simulated}. 

Lastly, the linear algebraic problems we consider are of great use in machine learning problems, which are frequently tackled using randomized approaches in order to avoid poor dimensional dependencies. As an example, randomized matrix sketching algorithms can substantially reduce the complexity of PCA or SVD for very large matrices. For instance, computing the top-$k$ singular vectors of a $d\times{}d$ matrix requires $kd^2$ time for traditional (deterministic) iterative methods, but can be reduced to $d^2 + dk^2/\epsilon^4$ using a randomized sketching approach \citep{clarkson2017low}, which can be much better for moderate $\epsilon$.


\paragraph{Related Work.} There is an extensive literature on algorithms for linear algebraic tasks such as solving linear systems and computing eigenvalues and eigenvectors, see for example the survey of \cite{sachdeva2014faster}. In the interest of brevity, we focus on the relevant lower bounds literature.  

The seminal work of \cite{nemirovskii1983problem} establishes lower bounds which apply only to \emph{deterministic} algorithms. These first order lower bounds enjoy essentially optimal dependence all relevant problem parameters, including dimension. However, these constructions are based on a so-called resisting oracle, and therefore \emph{do not} extend to  the randomized algorithms considered in this work. 

For randomized algorithms, the lower bounds of \cite{simchowitz2018tight} and \cite{simchowitz2018randomized} yield optimal dependence on the eigengap and condition number parameters. However, these bounds require the dimension to be polynomial large in these parameters, which translates into a suboptimal dimension-dependent lower bound of $\BigOmegaTil{d^{1/3}}$. 

A series of papers due to \citet{woodworth2016tight,woodworth2017lower} prove lower bounds for first order convex optimization algorithms which obtain optimal dependence on relevant parameters, but hold only in high dimensions. Furthermore, they are based on intricate, non-quadratic convex objectives which can effectively ``hide'' information in a way that linear algebraic instances cannot. Thus, they do not apply to the natural linear algebraic constructions that we consider. For high dimensional/low accuracy problems, there are also lower bounds for randomized algorithms that use higher order derivatives, see e.g. \citep{agarwal2017lower}. These, like the previously mentioned lower bounds, also only apply in high dimensions and imply dimension-dependent lower bounds like $\BigOmegaTil{d^{1/3}}$.

Finally, in concurrent, related work, \cite{sun2019querying} study numerous other linear algebraic primitives in the same matrix-vector product oracle setting. They use a similar approach to proving lower bounds for other problems and randomized algorithms, but do not address the fundamental problems of maximum eigenvalue computation and linear regression as we do.



\paragraph{Proof Techniques.}
One of the greatest strengths of our results is the simplicity of their proofs. In general, establishing query lower bounds which apply to randomized algorithms requires great care to control the amount of information accumulated by arbitrary, randomized, adaptive queries. Currently, the two dominant approaches are either (a) to construct complex problem instances that obfuscate information from any sequence of queries made \citep{woodworth2016tight}, or (b) reduce the problem to estimating of some hidden component \citep{simchowitz2018tight,simchowitz2018randomized}. The constructions for approach (a) are typically quite intricate, require high dimensions, and do not extend to linear algebraic problems. Approach (b) requires sophisticated information theoretic tools to control the rate at which information is accumulated. 

In contrast, our work leverages simple problems of a classic random matrix ensemble known as the \emph{Wishart} distribution \cite{anderson2010introduction}. 
In particular, our lower bound for maximum eigenvalue computation is witnessed by a very natural instance $\matM = \matW\matW^\top$ where the entries of $\matW$ are i.i.d.~Gaussian. This is plausibly a very benign instance as it is one of the simplest distributions over symmetric positive definite matrices that one might think of. 

The simplicity of the problem instance, and existing understanding of the distribution of the spectrum of Wishart matrices allows for concise, straightforward proofs.






\subsection{Notation} Let $\sphered := \{x \in \R^d:\|x\|_2 = 1\}$, $\Sym^d := \{M \in \R^{d \times d} : M = M^\top\}$ and $\Sympl^d := \{A \in \Sym^d: A \succ 0\}$. As a general rule, we use $M$ for matrices which arise in eigenvector problems, and $A$ for matrices which arise in least-squares problems.  For $M \in \Sym^d$, we let $\gap(M) := \frac{\lambda_1(M) - \lambda_2(M)}{\lambda_1(M)}$, and for $A \in \Sympl^d$, we set $\cond(A) := \frac{\lambda_1(A)}{\lambda_d(A)}$. We adopt the conventional notions $\BigOh{\cdot},\BigOmega{\cdot},\BigTheta{\cdot}$ as suppressing universal constants independent of dimension and problem parameters, let $\BigOhTil{\cdot},\BigOmegaTil{\cdot}$ suppress logarithmic factors, and let $g(x) = \BigOhSt{f(x)}$ denote a term which satisfies $g(x) \le c f(x)$ for a particular, unspecified, but sufficiently small constant $c$. We say a matrix is $s$ sparse if its number of nonzero entries is at most $s$. 

\newcommand{\condpar}{\mathtt{cond}}
\newcommand{\AlgEig}{\Alg_{\mathrm{eig}}}
\newcommand{\Quer}{\mathrm{Query}}

\newcommand{\gappar}{\mathtt{gap}}
\newcommand{\wi}{\mathsf{w}^{(i)}}
\renewcommand{\vi}{\mathsf{v}^{(i)}}
\newcommand{\EEA}{\mathrm{EigValueAlg}}
\newcommand{\EAA}{\mathrm{EigVecAlg}}
\newcommand{\QAA}{\mathrm{LinSysAlg}}

\newcommand{\cgap}{c_{\mathrm{gap}}}
\newcommand{\cquery}{c_1}
\newcommand{\cacc}{c_2}
\newcommand{\cprob}{c_{\mathrm{prob}}}
\newcommand{\ctopeig}{c_{\mathrm{eig}}}
\newcommand{\ccond}{c_{\mathrm{cond}}}
\newcommand{\dbar}{d_*}
\newcommand{\Mclassdbar}{\mathscr{M}_{\dbar}}

\section{Main Results}

We begin the section by stating a lower bound for the problem of eigenvalue estimation and eigenvector approximation via matrix-vector multiply queries. Via a standard reduction, this bound will imply a lower bound for solving linear systems via gradient-queries.  

We stress that, unlike prior lower bounds, our bounds for eigenvalue problems (resp. linear systems) both apply to arbitrary, randomized algorithms, \emph{and} capture the correct dependence on the eigengap (resp. condition number), all the way up to a $\Omega(d)$ (resp. $\widetilde{\Omega}(d)$) worst-case lower bound in $d$ dimensions. This worst-case lower bound is matched by the Lanczos \citep{musco2015randomized} and Conjugate Gradient methods (see, e.g.~\cite{trefethen1997numerical}), which, assuming infinite precision, efficiently recover the exact optimal solutions in at most $d$ queries. 

\subsection{Eigenvalue Problems}
Before introducing our results, we formalize the query model against which our lower bounds hold:
\begin{definition}[Eigenvalue and Eigenvector Algorithms] An eigenvalue approximation algorithm, or $\EEA$, is an algorithm $\Alg$ which interacts with an unknown matrix $M \in \PD $ via $T$ adaptive, randomized queries, $\wi = M\vi$, and returns an estimate $\lamhat$ of $\lambda_1(M)$.  An eigenvector approximation algorithm, or $\EAA$, operates in the same query model, but instead returns an estimate $\vhat \in \sphered$ of $v_1(M)$. We call $T := \Quer(\Alg)$ the \emph{query complexity} of $\Alg$. 
\end{definition}

We let $\Pr_{\matM\sim \calD,\Alg}$ denote the probability induced by running $\Alg$ when the input is a random instance $\matM$ drawn from a distribution $\calD$.
We now state our main query lower bound for $\EEA$'s, which we prove in Section~\ref{sec:proof_main_eig}. Our lower bound considers a distribution over symmetric matrices $M$ which are also PSD, to show that our lower bounds hold even under the most benign, and restrictive conditions:\footnote{Note that a lower bound on PSD matrices holds a fortiori for arbitrary symmetric and square matrices, since PSD are a subclass.}

\begin{theorem}[Lower Bound for Eigenvalue Estimation]\label{thm:main_eig_lb} There is a function  $d_0 :(0,1) \to \N$ such that the following holds. For any $\beta \in (0,1)$, ambient dimension $d \ge d_0(\beta)$, and every sparsity level $s \in [d_0(\beta),d]$, there exists a distribution $\calD = \calD(s,d,\beta)$ supported on $s^2$-sparse matrices in $\PDof{d}$ such that any $\EEA$ $\Alg$ with $\Quer(\Alg) \le (1-\beta)s$ satisfies
\begin{align*}
\Pr_{\matM \sim \calD,\Alg}\left[|\lamhat - \lambda_1(\matM)| \ge \frac{1}{20s^2}\right] \ge \Omega(\sqrt{\beta})
\end{align*}
Moreover, $\matM \sim \calD$ satisfies $\gap(\matM) \ge \Omega_{\beta}(1)/s^2$, and  $1 - \Omega_{\beta}(1) /s^2 \le \lambda_1(\matM) \le 1$ almost surely. Here, $\Omega_{\beta}(1)$ denotes a quantity lower bounded by a function of $\beta$, but not on $s$ or $d$. 
\end{theorem}
In particular, any algorithm requires $\Quer(\Alg) \ge \Omega(d)$ queries in ambient dimension $d$ to estimate $\lambda_1(\matM)$ up to error $\BigOhSt{d^{-2}}$ with constant probability, and in fact requires $(1-\LittleOh{1})d$ queries for a $1-\LittleOh{1}$ probability of error. 

\begin{remark}
The sparsity parameter in our construction can be used to encode the accuracy parameter $\epsilon$. Specifically, by setting the parameter $s = 1/ \sqrt{\epsilon}$, Theorem~\ref{thm:main_eig_lb} implies $1/\sqrt{\epsilon}$ queries are necessary for $\epsilon$-accuracy. Alternatively, choosing $s  \ge \Omega(\sqrt{\gap(\matM)})$, we obtain a gap-dependent bound requiring $\Omega(1/\sqrt{\gap(\matM)})$ queries for $\BigOh{\gap(\matM)}$ accuracy. Both bounds match the sharp lower bounds of \cite{simchowitz2018tight} up to logarithmic factors, while also capturing the correct worst-case query complexity for ambient dimension $d$, namely $\Omega(d)$. Moreover, our proof is considerably simpler. 
\end{remark}
\paragraph{Implications for sparsity.} For $s \ge \sqrt{d}$, our lower bound says that first order methods require $\Omega(s)$ queries to approximate the top eigenvalue of matrices $\matM$ with $\#nnz(\matM) = \Theta(s^2)$. Therefore, implementations of first-order methods based on standard matrix-vector multiplies cannot have complexity better than $\Omega(s^3)$ in the worst case. On the other hand, matrix inversion has runtime $d^{\omega}$, and is sufficient both for solving least squares, and for computing eigenvalues and eigenvectors \citep{garber2015fast}. Hence, for $s \in [d^{\omega/3},d]$, we see that matrix inversion  \emph{outperforms} first-order based methods. 

\paragraph{Approximating the top eigenvector.} As a corollary, we obtain the analogous lower bound for algorithms approximating the top eigenvector of a symmetric matrix, and, in particular, an $\Omega(s)$ query complexity lower bound for $\epsilon = \BigOhSt{s^{-2}}$-precision approximations: 
\begin{corollary}\label{cor:eig_approx} In the setting of Theorem~\ref{thm:main_eig_lb}, any $\EAA$ with $\Quer(\Alg) \le (1 - \beta)s - 1$ satisfies
\begin{align*}
\Pr_{\matM \sim \calD,\Alg}\left[\vhat^\top \matM \vhat \le \lambda_1(\matM)\left(1 - \frac{1}{20s^2}\right)\right] \ge \Omega(\sqrt{\beta})
\end{align*}
\end{corollary}
\begin{proof} For ease, set $\epsilon := \frac{1}{20s^2}$. Let  $\Quer(\Alg) \le (1-\beta)s - 1$, and let $\lamhat = \vhat^\top \matM \vhat$, which can be computed using at most $1$ additional query. Since $\lamhat \le \max_{v \in \sphered} v^\top \matM v = \lambda_1(\matM)$, and since $\lambda_1(\matM) \le 1$ we see that $\vhat^\top \matM \vhat \ge \lambda_1(\matM) - \epsilon$ only if $|\lamhat - \lambda_1(\matM)| \le \epsilon$. Thus, recalling $\lambda_1(\matM) \le 1$, we have $\Pr_{\calD,\Alg}\left[\vhat^\top \matM \vhat \ge \lambda_1(\matM)(1 - \epsilon) \right] \le \Pr_{\calD,\Alg}\left[\vhat^\top \matM \vhat \ge \lambda_1(\matM) - \epsilon \right] = \Pr_{\calD,\Alg}\left[ |\lamhat - \lambda_1(\matM)| \le \epsilon\right]$, which is at least $\Omega(\sqrt{\beta})$ by  Theorem~\ref{thm:main_eig_lb}.
\end{proof}
Again, the the sparsity $s$ can be used to encode the accuracy parameter $\epsilon$ via $\epsilon := \frac{1}{20s^2}$.







\subsection{Lower Bounds for Solving Linear Systems}

We now present our lower bounds for minimizing quadratic functions. We consider the following query model:

\begin{definition}[Gradient Query model for Linear System Solvers]\label{def:lin_sys} We say that $\Alg$ is an $\QAA$ if $\Alg$ is given initial point $x_0 \in \R^d$ and linear term $b \in \R^d$, and it interacts with an unknown symmetric matrix $A \in \PD $ via $T$ adaptive, randomized queries, $\wi = A\vi$, and returns an estimate $\xhat \in \R^d$ of $A^{-1}b$. Again, we call $T$ the \emph{query complexity} of $\Alg$.
\end{definition}
Defining the objective function $f_{A,b}(x) := \frac{1}{2}x^\top A x - b^\top x$, we see that the query model of Definition~\ref{def:lin_sys} is equivalent to being given a gradient query at $0$, $\nabla f_{A,b}(0) = b$, and making queries $\nabla f_{A,b}(\vi) = A\vi - b$. We shall use $\Pr_{\Alg,(x_0,b,A)}$ do denote probability induced by running the a $\QAA$ $\Alg$ on the instance $(x_0,b,A)$. Our lower bound in this model is as follows, stated in terms of the function suboptimality $\|\xhat - A^{-1}b\|_A^2 = f_{A,b}(\xhat) - \min_x f_{A,b}(x)$.

\begin{theorem}[Lower Bound for Linear System Solvers]\label{thm:quad_lb} Let $d_0\in \N$ be a universal constant. Then for all ambient dimensions $d\ge d_0$, and all $s \in \,[d_0^2,\,d^2]$,  any $\QAA$ $\Alg$ which satisfies the guarantee 
\begin{align}
\Pr_{\Alg,x_0,b,A}\left[\|\xhat - A^{-1}b\|_A^2 \le \frac{1}{e} \cdot \frac{ \lambda_1(A) \|x_0\|^2}{\,s^2}\right] \ge 1 - \frac{1}{e} \label{eq:to_contradict}
\end{align}
for all $(d + s^2)$-sparse matrices $M \in \PD$ with $\cond(M) \le \BigOh{s^2}$, and all $(x_0,b) \in \R^d \times \R^d$, must have query complexity at least
\begin{align*}
\Quer(\Alg) \ge  \BigOmega{ s \cdot (\log^{2+\log} s)^{-1}},
\end{align*}
where  $\log^{p+\log}(x) := (\log^p x) \cdot (\log \log x)$. 
\end{theorem}
In particular, any algorithm which ensures $\|\xhat - A^{-1}b\|_A^2 \le \BigOhSt{\frac{ \lambda_1(A) \|x_0\|^2}{ d^2}}$ with probability $1 - \frac{1}{e}$ requires $\BigOmegaTil{d}$-queries.  

\begin{remark}As with the eigenvector lower bounds, we can use the sparsity parameters to encode accuracy. Specifically, by $s = \sqrt{\cond}$, obtaining a function suboptimality $f_{A,b}(\xhat) - \min_x f_{A,b}(x) = \|\xhat - A^{-1}b\|_A^2 $ of $\BigOhSt{1/\cond}$ requires $\BigOmegaTil{\sqrt{\cond}}$ queries, matching known upper bounds achieved by the conjugate gradient method up to logarithmic factors \citep{trefethen1997numerical}, which in turn match information-theoretic lower bounds \citep{simchowitz2018randomized}. Moreover, these in turn can be converted into an \emph{minimax} lower bound by in turn selecting the condition number parameter as $\cond \propto 1/\epsilon$. Indeed, this implies that to ensure $f_{A,b}(\xhat) - \min_x f_{A,b}(x) \le \epsilon$, one requires $\BigOmegaTil{1/\sqrt{\epsilon}}$ queries, entailing the minimax rate up to logarithmic factors. 
\end{remark}

We prove Theorem~\ref{thm:quad_lb} by leveraging a well-known reduction from eigenvector approximation to minimizing quadratic functions, known as ``shift-and-invert'' \citep{saad2011numerical,pmlr-v48-garber16} . To state the result, we define a class of matrices to which the reduction applies:
\begin{align*}
\Mclass(\gappar,\alpha) := \left\{M \in \Sympl^d: \gap(M) \ge \gappar,~ |\lambda_1(M)-1|\le \alpha\, \gappar,~ \lambda_1(M) \in \left[\tfrac{1}{2},2\right]\right\},
\end{align*}
The term $\gappar$ corresponds to $\gap(M)$, whereas $\alpha$ measures to how close $\lambda_1(M)$ is to $1$. The rescaling to ensure $\lambda_1(M) \in \left[\tfrac{1}{2},2\right]$ is for simplicity, and more generally $\alpha$ corresponds to an approximate foreknowledge of $\lambda_1(M)$ (which is necessary to facilitate the reduction). We further note that the distribution $\calD(s,d,\beta)$ from Theorem~\ref{thm:main_eig_lb} satisfies, for some functions $\cgap(\cdot)$ and $\ctopeig(\cdot)$,
\begin{align}
\Pr_{\matM \sim \calD(s,d,\beta)}\left[\matM \in \Mclass\left(\frac{\cgap(\beta)}{s^2}, \,\frac{\ctopeig(\beta)}{\cgap(\beta)}\right) \right] = 1 \label{eq:Mclass_Dist}
\end{align}

With this definition in hand, we provide a precise guarantee for the reduction, which we prove in Appendix~\ref{sec:lin_redux_proof}:

\newcommand{\dmin}{d_{\min}}
\begin{proposition}[Eigenvector-to-Linear-System Reduction]\label{prop:eig_lin_reduc} Let $d \ge \dmin$ for a universal $\dmin$.  Fix a $\gappar \in (0,1)$, $\alpha > 0$, and suppose that $\Alg$ be a $\QAA$ which satisfies~\eqref{eq:to_contradict} with $\condpar := 1+\alpha + \frac{1}{\gappar}$ for all $A \in \PD$ with $\cond(A) \le \condpar$. Then, for any $\delta \in (0,1/e)$,  there exists an $\EAA$, $\AlgEig$, which satisfies
\begin{align*}
 \Pr_{\AlgEig,M}\left[\vhat^\top M \vhat \ge (1 - c\gappar)\lambda_1(M) \right] \ge 1 - \delta, \quad \forall M \in \Mclass\left(\gappar,\alpha\right) 
\end{align*}
with query complexity at most 
\begin{align*}
\Quer(\AlgEig) \le \Quer(\Alg) \cdot \BigOhPar{\alpha}{(\log \frac{1}{\delta}) \cdot \log^{2+\log} \frac{d}{\min\{c\gappar,\, 1\}}},
\end{align*} where $\BigOhPar{\alpha}{\cdot}$ hides multiplicative and additive constants depending on $\alpha$.
\end{proposition}

We can now prove Theorem~\ref{thm:quad_lb} by combining Proposition~\ref{prop:eig_lin_reduc} and Theorem~\ref{thm:main_eig_lb}:

%
\begin{proof}
We shall prove the theorem in the regime where $s = d$. The general $s$ case is attained by embedding an instance of dimension $s$ into dimension $d$, as in the proof of Theorem~\ref{thm:main_eig_lb}, and is deferred to Appendix~\ref{sec:arbitrary_condition}. 

To begin, let $\beta = \frac{1}{2}$ (any constant in $(0,1)$ suffices); throughout, $\beta$ will be a universal constant, rather than a problem parameter.
Next, fix an ambient dimension $d \ge d_0 := \dmin \vee d_0(\beta)$, where $d_0$ is from Theorem~\ref{thm:main_eig_lb}, and $\dmin$ from Proposition~\ref{prop:eig_lin_reduc}. 

Lastly, let $\gappar := \frac{\cgap(\beta)}{d^2}$ and $\alpha := \frac{\ctopeig(\beta)}{\cgap(\beta)}$, where $\cgap(\cdot)$ and $\ctopeig(\cdot)$ are as in \eqref{eq:Mclass_Dist}. Let $\matM \sim \calD(d,d,\beta)$. Then,~\eqref{eq:Mclass_Dist} ensures $\matM \in \Mclass(\gappar, \alpha)$ with probability $1$. For the sake of contradiction, suppose that $\Alg$ is a $\QAA$ which satisfies the guarantee of~\eqref{eq:to_contradict} for all 
\begin{align*}
A: \cond(A) \le 1 + \alpha + \frac{1}{\gappar} =  \BigOh{d^2}.
\end{align*} 

Then, for the universal constant $c := \frac{1}{20\cgap(\beta)}$, there exists an $\EAA$ $\AlgEig$ which satisfies, for all $M \in \Mclass\left(\gappar,\alpha\right)$,
\begin{align*}
 \Pr_{\AlgEig,M}\left[\vhat^\top M \vhat \ge \left(1 - \frac{1}{20d^2}\right)\lambda_1(M) \right] &= \Pr_{\AlgEig,M}\left[\vhat^\top M \vhat \ge (1 - c\,\gappar)\lambda_1(M) \right] \ge 1 - \Omega(\sqrt{\beta}), 
\end{align*}
whose query complexity $\Quer(\AlgEig)$ is bounded by
\begin{align*}
 \Quer(\Alg) \cdot \BigOhPar{\alpha}{(\log \tfrac{1}{\Omega(\sqrt{\beta}) \wedge \,e}) \cdot \log^{2+\log} \tfrac{d}{1 \wedge c\,\gappar,1}} = \Quer(\Alg) \cdot \BigOh{ \log^{2+\log} d},
\end{align*}
where we use $\gappar = \Theta(d^2)$, and that $\alpha,c,\beta,\Omega(\beta)$ depend on universal constants, and not on the choice of dimension $d$.  By Theorem~\ref{thm:main_eig_lb}, we must have $\Quer(\AlgEig) \ge d/2$, whence
\begin{align*}
\Quer(\Alg) \ge \frac{\cquery}{\sqrt{\gappar}} \cdot  \BigOmega{ (\log^{2+\log} d)^{-1}} = \BigOmega{d (\log^{2+\log} d)^{-1}}.
\end{align*}

\end{proof}


\section{Proof of Theorem~\ref{thm:main_eig_lb}\label{sec:proof_main_eig}}
\newcommand{\distconv}{\overset{\mathrm{dist}}{\to}}
\newcommand{\wishmatd}[1]{\mathbf{M}^{(d)}}
\newcommand{\wishmat}{\mathbf{W}}
\newcommand{\Zsf}{\mathsf{Z}}
\newcommand{\perr}{p_{\mathrm{err}}}
\newcommand{\Xmat}{\mathbf{X}}
\newcommand{\matLam}{\boldsymbol{\Lambda}}
\newcommand{\matx}{\mathbf{x}}
\newcommand{\matz}{\mathbf{z}}
\newcommand{\dnot}{\mathsf{d}_0}

\newcommand{\vsf}{\mathsf{v}}
\newcommand{\wsf}{\mathsf{w}}

\newcommand{\ortho}{\mathscr{O}}
\newcommand{\matXtil}{\widetilde{\matX}}
\newcommand{\matY}{\mathbf{Y}}

\newcommand{\matutil}{\widetilde{\matU}}
\newcommand{\Wigner}{\mathrm{Wigner}}
\newcommand{\Wishart}{\mathrm{Wishart}}

\newcommand{\matYtil}{\widetilde{\matY}}
\newcommand{\matxtil}{\widetilde{\matx}}
\newcommand{\matytil}{\widetilde{\maty}}
\newcommand{\tuple}{\mathsf{Z}}
\newcommand{\matr}{\mathbf{r}}
\newcommand{\wtil}{\widetilde{w}}
\newcommand{\mato}{\mathbf{o}}
\newcommand{\matO}{\mathbf{O}}
\newcommand{\matR}{\mathbf{R}}
\newcommand{\matV}{\mathbf{V}}
\newcommand{\Xmatd}{\mathbf{X}^{(d)}}
\newcommand{\Xmatdtop}{\mathbf{X}^{(d)\top}}

The proof of Theorem~\ref{thm:main_eig_lb} follows by deriving a lower bound for the problem of estimating the least eigenvalue of a classical random matrix ensemble known as the (standard) \emph{Wishart} matrices:
\begin{definition}[Wishart Distribution] We write $\wishmat \sim \Wishart(d)$ to denote a random matrix with the distribution $\wishmat \overset{d}{=} \Xmat\Xmat^\top$, where $\Xmat \in \R^{d \times d}$ and $\Xmat_{i,j} \iidsim \calN(0,\frac{1}{d}I) $.
\end{definition}

\newcommand{\dimfunc}{\mathsf{d}}
\newcommand{\lamhatmin}{\lamhat_{\min}}
\newcommand{\cprobwish}{\mathrm{c}_{\mathrm{wish}}}

We now state our main technical contribution, which lower bounds the number of queries required for estimation the smallest eigenvalue of a matrix $\wishmat \sim \Wishart(d)$:
\begin{theorem}[Lower Bound for Wishart Eigenvalue Estimation]\label{thm:main_wish}  There exists a universal constant $p_0$ and function $\dimfunc: (0,1) \to \N$ such that the following holds: for all $\beta \in (0,1)$, and all $d \ge \dimfunc(\beta)$, we have that $\wishmat \sim \Wishart(d)$ satisfies 
\begin{enumerate}
	\item[(a)] Any algorithm $\Alg$ which makes $T \le (1-\beta) d$ adaptively chosen queries, and returns an estimate $\lamhat_{\min}$ of $\lambda_{\min}(\wishmat)$ satisfies
	\begin{align*}
	\Pr_{\wishmat,\Alg}\left[|\lamhatmin - \lambda_{\min}(\wishmat)| \ge \frac{1}{4d^2}\right] \ge \cprobwish \sqrt{\beta}.
	\end{align*} 
	\item[(b)] There exists constants $C_1(\beta)$ and $C_2(\beta)$ such that 
	\iftoggle{colt}
	{
		\begin{multline*}
		\Pr_{\wishmat}\left[\{\lambda_d(\wishmat) \le C_1(\beta)d^{-2}\} \cap \{\lambda_{d-1}(\wishmat) - \lambda_d(\wishmat) \ge C_2(\beta)d^{-2}\} \cap \{\|\wishmat\| < 5\}  \right] \\
		\ge 1 - \tfrac{\cprobwish \sqrt{\beta}}{2}.
		\end{multline*}
	}
	{
		\begin{align*}
		\Pr_{\wishmat}\left[\{\lambda_d(\wishmat) \le C_1(\beta)d^{-2}\} \cap \{\lambda_{d-1}(\wishmat) - \lambda_d(\wishmat) \ge C_2(\beta)d^{-2}\} \cap \{\|\wishmat\| < 5\}  \right] 
		\ge 1 - \tfrac{\cprobwish \sqrt{\beta}}{2}.
		\end{align*}
	}
\end{enumerate}
\end{theorem}
Note that, by taking $\beta = \LittleOh{1}$,  Theorem~\ref{thm:main_wish} in fact demonstrates that $(1-\LittleOh{1}))d$ queries are required for an $\LittleOh{1}$ probability of failure, showing that no nontrivial improvements can be achieved.

\begin{proof}[Proof of Theorem~\ref{thm:main_eig_lb}] Fix  $\beta \in (0,1)$, and let $d_0(\cdot) = \dimfunc(\cdot)$ denote the function from Theorem~\ref{thm:main_wish}. For $s \ge d_0(\beta)$, let $\wishmat \sim \Wishart(s)$, and define 
\begin{align*}
\matM = \begin{bmatrix} I_{s \times s}  - \frac{1}{5}\wishmat & 0 \\
0 & 0 \end{bmatrix} \in \R^{d \times d}.
\end{align*}
By construction, $\matM$ has sparsity $s^2$. 
Let us denote the event of part (b) of Theorem~\ref{thm:main_wish} $\calE$. Then $\calE$ occurs with with probability $1 - \tfrac{\cprobwish \sqrt{\beta}}{2}.$. On $\calE$, we further have 
\begin{itemize}
\item $0 \preceq \matM \preceq 1$
\item $\gap(\matM) = \frac{\lambda_1(\matM) - \lambda_2(\matM)}{\lambda_1(\matM)} \ge \frac{d^{-2}}{5}C_2(\beta)= \Omega_{\beta}(1) \cdot d^{-2}$
\item $|\lambda_1(\matM) - 1| = \frac{1}{5}\lambda_{\min}(\matW) \le \Omega_{\beta}(1) d^{-2}$
\end{itemize}
Now consider an estimator $\lamhat$ of $\lambda_{\max}(\matM)$. By considering the induced estimator $\lamhatmin := 5(1- \lamhat)$ of $\lambda_{\min}(\wishmat)$, part (a) of Theorem~\ref{thm:main_wish} and a union bound implies that

\begin{align*}
\Pr_{\matM,\Alg}\left[\{|\lamhat - \lambda_{\max}(\matM)| \le \frac{1}{20d^2}\} \mid \calE \right] &=\Pr_{\wishmat,\Alg}\left[\{|\lamhatmin - \lambda_{\min}(\wishmat))| \le \frac{1}{4d^2}\} \mid \calE \right] \\
&=\Pr_{\wishmat,\Alg}\left[\{|\lamhatmin - \lambda_{\min}(\wishmat))| \le \frac{1}{4d^2}\} \cap \calE \right] \\
&\ge\Pr_{\wishmat,\Alg}\left[\{|\lamhatmin - \lambda_{\min}(\wishmat))| \le \frac{1}{4d^2}\}\right] -  \Pr_{\wishmat,\Alg}\left[\calE \right] \\
&\overset{(i)}{\ge} \cprobwish \sqrt{\beta} - \frac{1}{2}\cprobwish \sqrt{\beta} = \frac{1}{2}\cprobwish \sqrt{\beta}
\end{align*}
where $(i)$ uses Theorem~\ref{thm:main_eig_lb}.
Hence, let $\calD$ denote the distribution of $\matM$ conditioned on $\calE$, any $\EEA$ $\Alg$ with $\Quer(\Alg) \le (1 - \beta)d$ queries satisfies 
\begin{align*}
\Pr_{\matM \sim \calD,\Alg}\left[\{|\lamhat - \lambda_{\max}(\matM)| \le \frac{1}{20d^2}\} \right] \ge \frac{1}{2}\cprobwish \sqrt{\beta} = \Omega(\sqrt{\beta}).
\end{align*} 
\end{proof}

\subsection{Proof of Theorem~\ref{thm:main_wish}}
We begin the proof of Theorem~\ref{thm:main_wish} by collecting some useful facts from the literature regarding the asymptotic distribution of Wishart spectra. 
\begin{lemma}[Facts about Wishart Matrices]\label{lem:wishart_facts} Let $(\matz_d^{(d)},\matz_{d-1}^{(d)}) \in R^2$ denote random variables with the (joint) law of $(d^2\lambda_d(\wishmat^{(d)}),d^2\lambda_{d-1}(\wishmat^{(d)}))$, where $\wishmat^{(d)} \sim \Wishart(d)$. The following are true:
\begin{enumerate}
	\item   $(\matz_d^{(d)},\matz_{d-1}^{(d)})$ converge in distribution to $\calD$ distribution with $\Pr_{(\matz_d,\matz_{d-1})\sim \calD}[0 < \matz_d < \matz_{d-1}] = 1$ \emph{(\citet[Theorem 1]{ramirez2009diffusion})}.
	\item $\matz_d$ has the density $f(x) = \I(x \ge 0) \cdot \frac{x^{-1/2} +1}{2}e^{-(x/2 + \sqrt{x})}$ \emph{(e.g. \citet[Page 3]{shen2001singular})}
\end{enumerate}
Moreover, for any $\epsilon > 0$, $\lim_{d \to \infty} \Pr_{\wishmat \sim \Wishart(d)} [\|\wishmat\|_{\op} \ge 4 + \epsilon] = 0 $ \emph{(e.g. \citet[Exercise 2.1.18]{anderson2010introduction})}
\end{lemma}

\newcommand{\dreg}{d_{\mathrm{reg}}}
\newcommand{\ddens}{d_{\mathrm{dens}}}
We note that we use $(\matz_d^{(d)},\matz_{d-1}^{(d)})$ for the normalized (by $d^2$) eigenvalues of $\wishmat^{(d)}$. We convert these asymptotic guarantees into quantitative ones (proof in Section~\ref{sec:cor_wishart_proof}).
\begin{corollary}[Non-Asymptotic Properties]\label{cor:wishart_facts} There exists a maps $\dreg,\ddens:(0,1) \to \N$, functions $C_1,C_2:(0,1) \to \R_{>0}$, and a universal constant $p_0$ such that the following holds: for any $\delta \in (0,1)$ and $d \ge \dreg(\delta)$,
\begin{align}
\Pr_{\wishmat \sim \Wishart(d)}\left[   \left\{\matz_{d-1}^{(d)} - \matz_d^{(d)} \ge C_2(\delta)\right\} \cap\left\{ \matz_d^{(d)} \le C_1(\delta)\right\} \cap \left\{\|\wishmat\|_{\op} \le 5\right\} \right] \ge 1 - \delta  \label{eq:good_event}
\end{align}
Moreover, for any $\alpha \in (0,1)$ and $d \ge \ddens(\alpha)$, $\Pr[\lambda_d(\wishmat^{(d)}) \ge d^{-2}] \ge p_0$, while $\text{and} \quad \Pr[\lambda_d(\wishmat^{(d)}) \le \alpha^2 d^{-2}] \ge p_0 \alpha$.
\end{corollary}

We now show establish, in an appropriate basis, $\wishmat \sim \Wishart(d)$ admits a useful block decomposition:
\newcommand{\Wcorner}{\widetilde{\matW}}
\begin{lemma}\label{lem:independent-component}
Let $\wishmat \sim \Wishart(d)$. Then, for any sequence of queries $v^{(1)},\dots,v^{(T)}$ and responses $w^{(1)},\dots,w^{(T)}$, there exists a rotation matrix $\matV$ constructed solely as a function of $v^{(1)},\dots,v^{(T)}$ such that the matrix $\matV\wishmat\matV^\top$ can be written 
\begin{align*}
\matV\wishmat\matV^\top =
\begin{bmatrix}
Y_1Y_1^\top & Y_1Y_2^\top \\
Y_2Y_1^\top & Y_2Y_2^\top + \Wcorner
\end{bmatrix}
\end{align*}
where $\Wcorner$ conditioned on the event $\vsf^{(1)}=v^{(1)},\dots,\vsf^{(T)}=v^{(T)}$, $\wsf^{(1)}=w^{(1)},\dots,\wsf^{(T)}=w^{(T)}$ satisfies $(\frac{d}{d-T}) \cdot \Wcorner \sim \Wishart(d-T)$ distribution.
\end{lemma}

The above lemma is proven in in \iftoggle{colt}{Appendix}{Section}~\ref{sec:lem-independent}. The upshot of the lemma is that after $T$ queries, there is still a portion $\Wcorner$ of $\wishmat$ that remains unknown to the query algorithm. We now show that this unknown portion exerts significant influence on the smallest eigenvalue of $\wishmat$. Specifically, the following technical lemma implies that $\lambda_{\min}(\wishmat) = \lambda_{\min}(\matV\wishmat\matV^\top) \le \lambda_{\min}(\Wcorner)$:
\begin{lemma}\label{lem:eigenvalue-of-independent-component}
For $A \in \R^{T\times T}$, $B \in \R^{(d-T) \times T}$, and symmetric $\widetilde{W} \in \R^{(d-T)\times(d-T)}$, let
\begin{align*}
M =
\begin{bmatrix}
AA^\top & AB^\top \\
BA^\top & BB^\top + W
\end{bmatrix}
\end{align*}
then $\lambda_{\min}(M) \leq \lambda_{\min}(W)$.
\end{lemma}
\begin{proof}
Let $v \in \R^{d-T}$ such that $\norm{v} = 1$ and $v^\top W v = \lambda_{\min}(W)$. Define $z = \begin{bmatrix} -A^{-\top}B^\top v \\ v \end{bmatrix}$. Then
\begin{equation*}
z^\top M z 
= z^\top\begin{bmatrix}
-AB^\top v + AB^\top v \\
-BB^\top v + BB^\top v + Wv
\end{bmatrix}
= 
v^\top W v = \lambda_{\min}(W)
\end{equation*}
Therefore, $\lambda_{\min}(M) \leq \lambda_{\min}(W)$.
\end{proof}

With all the above ingredients in place, we are now ready to complete the proof of Theorem~\ref{thm:main_wish}:
\begin{proof}[Proof of Theorem~\ref{thm:main_wish}] Let $T \le (1-\beta) d$, and let $t = \frac{1}{2d^2}$, $\epsilon = \frac{1}{4d^2}$. Moreover, let $\Zsf := \{\vsf^{(1)},\dots,\vsf^{(T)},\wsf^{(1)},\dots,\wsf^{(T)}\}$ encode the query-response information, and let $\Wcorner$ denote the matrix from Lemma~\ref{lem:independent-component}. Finally, define the error probability
\begin{align*}
\perr  := \Pr[|\lamhatmin - \lambda_{\min}(\wishmat)| \ge \epsilon].
\end{align*}
We can now lower bound the probability of error by lower bounding the probability that the algorithm ouputs an estimate $\lamhatmin$ above a threshold $t$, while the corner matrix $\Wcorner$ has smallest eignvalue below $t - \epsilon$. We can then decouple the probability of these events using independence of $\Wcorner$ conditioned on the queries
\begin{align}
\perr \ge\Pr[ \{\lamhatmin \ge t\} \cap \{t - \epsilon \ge \lambda_{\min}(\wishmat)\}] &\overset{(i)}{\ge} \Pr[ \{\lamhatmin \ge t\} \cap \{t - \epsilon \ge \lambda_{\min}(\Wcorner)\}] \nonumber\\
&= \Exp\left[\Pr[\{\lamhatmin \ge t\} \cap \{t - \epsilon \ge \lambda_{\min}(\Wcorner)\} \mid \Zsf]\right] \nonumber\\
&\overset{(ii)}{=} \Exp\left[\Pr[\lamhatmin \ge t \mid \Zsf] \cdot \Pr[\lambda_{\min}(\Wcorner) \le t-\epsilon \mid \Zsf]\right] \nonumber\\
&\overset{(iii)}{=}  \Pr[\lambda_{\min}(\Wcorner) \le t - \epsilon ] \cdot \Pr[\lamhatmin \ge t], \label{eq:p_err_first}
\end{align}
where $(i)$ uses Lemma~\ref{lem:eigenvalue-of-independent-component},
$(ii)$ uses the fact that, conditioned on the queries in $\Zsf$, $\lamhatmin$ depends only on the internal randomness of the algorithm, but not on the corner matrix $\Wcorner$, and $(iii)$ uses the the fact that $\Wcorner$ has a Wishart distribution conditioned on $\Zsf$, and thus $\lambda_{\min}(\Wcorner)$ is independent of $\Zsf$. On the other hand, we have
\begin{align*}
\perr &\ge \Pr[ \{\lamhatmin < t\} \cap \{ \lambda_{\min}(\wishmat) \ge t + \epsilon\}] \ge  \Pr[\lambda_{\min}(\wishmat) \ge t + \epsilon] - \Pr[\lamhatmin \ge t],
\end{align*}
so that $ \Pr[\lamhatmin \ge t] \ge   \Pr[\lambda_{\min}(\wishmat) \ge t + \epsilon] - \perr $. Together with \eqref{eq:p_err_first}, this implies
\begin{align*}
\perr &\ge \Pr[\lambda_{\min}(\Wcorner) \le t - \epsilon ] \cdot \Pr[\lamhatmin > t] \\
&\ge \Pr[\lambda_{\min}(\Wcorner) \le t - \epsilon ] \cdot \left(  \Pr[\lambda_{\min}(\wishmat) \ge t + \epsilon] - \perr \right).
\end{align*}
Performing some algebra, and recalling $\epsilon = \frac{1}{4d^2}$, $t = 2\epsilon$, 
\begin{align*}
\perr \ge \frac{\Pr[\lambda_{\min}(\Wcorner) \le t- \epsilon] \cdot \Pr[\lambda_{\min}(\wishmat) \ge t + \epsilon]}{1 + \Pr[\lambda_{\min}(\Wcorner) \le t - \epsilon]} \ge  \frac{\Pr[d^2\lambda_{\min}(\Wcorner) \le \frac{1}{4}] \cdot \Pr[d^2\lambda_{\min}(\wishmat) \ge 1]}{2}.
\end{align*}
 Next, we develop
\begin{align*}
\Pr\left[d^2\lambda_{\min}(\Wcorner) \le \frac{1}{4}\right] &= 
 \Pr\left[ (d-T)^2 \frac{d}{d-T} \cdot  \lambda_{\min}\left( \frac{d}{d-T} \Wcorner\right) \le \frac{1}{4}\right] \\
&\overset{(i)}{=} \Pr_{\Wcorner' \sim \cdot \Wishart(T-d)}\left[(d-T)^2 \frac{d}{d-T} \lambda_{\min}(\Wcorner') \le \frac{1}{4}\right] \\
&= \Pr_{\Wcorner' \sim \cdot \Wishart(T-d)}\left[(d-T)^2  \lambda_{\min}(\Wcorner') \le \frac{T-d}{4d}\right]\\
&\overset{(ii)}{\ge} \Pr_{\Wcorner' \sim \cdot \Wishart(T-d)}\left[(d-T)^2  \lambda_{\min}(\Wcorner') \le \frac{\beta}{4}\right],
\end{align*}
where $(i)$ uses that $\frac{d}{d-T} \Wcorner \sim \Wishart(T-d)$ by Lemma \ref{lem:independent-component}, and $(ii)$ uses $T \le (1-\beta) d$. 

Define $d_T := d-T \ge \beta d$. Then, if $d_T \ge \ddens(\beta/4)$, where $\ddens$ is the function from Corollary~\ref{cor:wishart_facts}, Corollary~\ref{cor:wishart_facts} yields the existence of constant $p_0 > 0$ such that $ \Pr_{\Wcorner' \sim \Wishart(T-d)}[(T-d)^2\lambda_{\min}(\Wcorner') \le \frac{\beta}{4}] \ge \sqrt{\beta/4}p_0$, and simultaneously, since $d \ge d_T$, such that $ \Pr[d^2\lambda_{\min}(\wishmat) \ge 1] \ge p_0$. Hence, for $d_T \ge \ddens(\beta/4)$  and for $\cprobwish= \frac{p_0^2}{4}$ (which does not depend on $T,d,\beta$), we find
\begin{align*}
\perr  &\ge  \frac{\Pr[\lambda_{\min}(\Wcorner) \le \frac{1}{2}] \cdot \Pr[\lambda_{\min}(\wishmat) \ge 1]}{2}\\
&\ge  \frac{ \Pr_{\widetilde{\wishmat} \sim \Wishart(T-d)}\left[\lambda_{\min}(\Wcorner) \le \frac{\beta}{4}\right] \cdot \Pr[\lambda_{\min}(\wishmat) \ge 1]}{2}\\
&\ge \frac{\sqrt{\beta/4}p_0 \cdot p_0}{2} = \frac{p_0^2}{4}\sqrt{\beta} = \cprobwish\sqrt{\beta}.
\end{align*}
To conclude, we need to ensure $d_T \ge  \ddens(\frac{\beta}{4})$, where again $\ddens$ is the function from Corollary~\ref{cor:wishart_facts}. Since $d_T = T - d \ge \beta d$, it suffices that $d \ge \dimfunc(\beta) := \beta^{-1}\ddens(\beta/4)$. This concludes the proof.
\end{proof}


\iftoggle{colt}{}{

\section{Omitted Proofs from Section~\ref{sec:proof_main_eig}}
\subsection{Proof of Corollary~\ref{cor:wishart_facts} \label{sec:cor_wishart_proof}}

For the first point, fix a $\delta \in (0,1)$. Then by Lemma~\ref{lem:wishart_facts}, the limiting normalized distributions of the eigenvalues $(\matz_d,\matz_{d-1})$ satisfy $\Pr[ \{\matz_d \ge C_1(\delta)\} \cap \{\matz_{d-1} - \matz_d \le C_2(\delta)\}] \le \frac{\delta}{3}$ for appropriate constants $C_1(\delta),C_2(\delta)$. By convergence in distribution, and the fact that $\{z_1\ge C_1(\delta)\} \cap \{z_2 - z_1 \le C_2(\delta)\}$ corresponds to the event that $(z_1,z_2)$ lie in a closed set, $\lim_{d \to \infty}\Pr[ \{\matz_d^{(d)} \ge C_1(\delta)\} \cap \{\matz_{d-1}^{(d)} - \matz_d^{(d)} \le C_2(\delta)\}] \le \frac{\delta}{3}$, so that for all $d$ sufficently large as a function of $\delta$, $\Pr[ \{\matz_d^{(d)} \ge C_1(\delta)\} \cap \{\matz_{d-1}^{(d)} - \matz_d^{(d)} \le C_2(\delta)\}] \le \frac{2\delta}{3}$. Finally, for all $d$ sufficiently large as a function of $\delta$, we have $\Pr[\lambda_{\max}(\wishmat) \ge 5] \le \frac{\delta}{3}$. The result now follows from a union bound.

For the second point, we use that the limiting distribution of $d^2\lambda_1(\wishmat)$ has density $f(x) = \I(x \ge 0) \cdot \frac{x^{-1/2} +1}{2}e^{-(x/2 + \sqrt{x})}$. Recalling the notation $\matz_d$ for a random variable with said limiting distribution, integrating the density shows that there exists a constant $p$ such that, for all for all $\alpha \in (0,1)$, $\Pr[\matz_d \ge 1] \ge p$ and $\Pr[\matz_d \le \alpha^2 ] \ge p \alpha$. The bound now follows by invoking convergence in distribution and setting, say $p_0= p/2$.

\subsection{Proof of Lemma~\ref{lem:independent-component}\label{sec:lem-independent}}
    
The key idea of the lemma is that orthogonal components of a Gaussian vector are independent of each other. Specifically, if $x \sim \mathcal{N}(0,\sigma^2 I)$, then for any $a,b$ with $\langle a,b\rangle = 0$, we have that $\langle a,x\rangle$ and $\langle b,x\rangle$ are independent of each other. The proof of the lemma essentially revolves around constructing appropriate rotation matrices such that the relevant ``orthogonal component'' of the matrix $\matX\matX^\top$ is located in the bottom-right corner. 
    
Without loss of generality, we can take the query vectors $v^{(1)},\dots,v^{(T)}$ to be orthonormal. This is because the response to any sequence of queries can be calculated from the responses to a sequence of orthonormal queries, so for any algorithm that does \emph{not} use orthonormal queries, there is another, equivalent algorithm that does.

We define a sequence of matrices $\matV_1,\dots,\matV_T$ such that for each $t \leq T$, the product $\matV_{1:t} = \matV_t\matV_{t-1}\dots\matV_1$ is an orthonormal matrix whose first $t$ rows are $v^{(1)},\dots,v^{(t)}$. The remaining rows of $\matV_{1:t}$ are arbitrary, but we choose them deterministically so that $\matV_{1:t}$ is measurable with respect to $v^{(1)},\dots,v^{(t)}$. To accomplish this, we note that since $\matV_{1:t-1}$ is orthonormal and thus full-rank, we can set $\matV_t$ to be a matrix whose first $t-1$ rows are the standard basis vectors $e_1,\dots,e_{t-1}$, whose $t^{\mathrm{th}}$ row is the solution to $\matV_t[t,:] \matV_{1:t-1} = {v^{(t)}}^\top$, and whose remaining rows are chosen (deterministically) so that $\matV_{1:t}$ is orthonormal.

We also define a sequence of orthonormal matrices $\matR_1,\dots,\matR_T$ such that for each $t \leq T$, the first $t$ columns of $\matR_t$ form an orthonormal basis for the first $t$ rows of $\matV_{1:t}\matX \matR_{1:t-1}$, where $\matR_{1:t-1} = \matR_1\matR_2\dots\matR_{t-1}$.\footnote{The matrix $\matV_{1:t} \matX \matR_{1:t-1}$ is full rank with probability 1, but in case the span of its first $t$ rows has dimension less than $t$, the columns of $\matR_t$ can be chosen as an orthonormal basis for any $t$-dimensional subspace that contains the span of the first $t$ rows.} As with $\matV_1,\dots,\matV_t$, the remaining columns of $\matR_t$ are arbitrary, but we choose them deterministically as a function of the first $t$ rows of $\matV_{1:t}\matX \matR_{1:t-1}$.

Let $\matV_{1:t}^{\parallel}$ denote the first $t$ rows of $\matV_{1:t}$, and let $\matV_{1:t}^{\perp}$ denote the remaining $d-t$ rows. Similarly, let $\matR_{1:t}^{\parallel}$ denote the first $t$ columns of $\matR_{1:t}$, and let $\matR_{1:t}^{\perp}$ denote the remaining $d-t$ columns. Then, for any $t \leq T$ we can decompose
\begin{equation}\label{eq:decomposition-of-wishart}
\pars{\matV_{1:t}\matX\matR_{1:t}}\pars{\matV_{1:t}\matX\matR_{1:t}}^\top\\
= 
\begin{bmatrix}
\pars{\matV_{1:t}^\parallel \matX \matR_{1:t}^\parallel}\pars{\matV_{1:t}^\parallel \matX \matR_{1:t}^\parallel}^\top & 
\pars{\matV_{1:t}^\parallel \matX \matR_{1:t}^\parallel}\pars{\matV_{1:t}^\perp \matX \matR_{1:t}^\parallel}^\top \\
\pars{\matV_{1:t}^\perp \matX \matR_{1:t}^\parallel}\pars{\matV_{1:t}^\parallel \matX \matR_{1:t}^\parallel}^\top &
\begin{array}{c}
\pars{\matV_{1:t}^\perp \matX \matR_{1:t}^\parallel}\pars{\matV_{1:t}^\perp \matX \matR_{1:t}^\parallel}^\top \\ +
\pars{\matV_{1:t}^\perp \matX \matR_{1:t}^\perp}\pars{\matV_{1:t}^\perp \matX \matR_{1:t}^\perp}^\top
\end{array}
\end{bmatrix}
\end{equation}
Here, we used that $\matR_t$ is orthonormal, so its final $d-t$ rows are orthogonal to the first $t$ rows, which span the first $t$ rows of $\matV_{1:t}\matX \matR_{1:t-1}$, i.e., $\matV_{1:t}^\parallel \matX \matR_{1:t-1}$.

	
We will now, finally, show by induction that for each $T$, conditioned on the queries and observations $v^{(1)},\dots,v^{(T)}$ and $w^{(1)},\dots,w^{(T)}$, the matrix $\matV_{1:T}^\perp \matX \matR_{1:T}^\perp$ has independent $\mathcal{N}(0,\frac{1}{d})$ entries, which proves the lemma.

\paragraph{Base case:}

Conditioning on the first query $v^{(1)}$ fixes the matrix $\matV_1$, which is independent of $\matX$. Since each column of $\matX$ is an independent Gaussian vector with distribution $\mathcal{N}(0,\frac{1}{d}I)$, since the rows of $\matV_1$ are orthonormal, and since orthogonal components of Gaussian vectors are independent, it follows that the entries of $\matV_1 \matX$ are independent and have distribution $\mathcal{N}(0,\frac{1}{d})$, conditional on $v^{(1)}$. 

Similarly, the matrix $\matR_1$ is measurable with respect to $\matV_1^\parallel \matX$. Just now, we showed that $\matV_1^\parallel \matX$---i.e.~the first row of $\matV_1 \matX$---is independent of $\matV_1^\perp \matX$---i.e.~all the other rows of $\matV_1\matX$. Therefore, $\matR_1$ is independent of $\matV_1^\perp \matX$ and because, again, orthogonal components of Gaussian vectors are independent, it therefore follows that the entries of $\matV_1^\perp \matX \matR_1$ also have independent $\mathcal{N}(0,\frac{1}{d})$ entries, conditional on $v^{(1)}$. This also means that 
\begin{equation}
\matV_1 \matX \matR_1^\parallel =
\begin{bmatrix}
1 \\
\matV_1^\perp \matX \matR_1^\parallel
\end{bmatrix}
\end{equation}
is independent of $\matV_1^\perp \matX \matR_1^\perp$. Finally, in light of \eqref{eq:decomposition-of-wishart}, we have
\begin{equation}
w^{(1)} 
= \matX\matX^\top v^{(1)}
= \matV_1^\top \matV_1 \matX \matR_1 \matR_1^\top \matX^\top \matV_1^\top e_1
= \matV_1^\top \begin{bmatrix}
\pars{\matV_1^\parallel \matX \matR_1^\parallel}^2 \\
\pars{\matV_1^\parallel \matX \matR_1^\parallel} 
\matV_1^\perp \matX \matR_1^\parallel
\end{bmatrix}
\end{equation}
We now condition on both the query $v^{(1)}$ and the observation $w^{(1)}$, which is measurable with respect to $\matV_1$ and $\matV_1\matX \matR_1^\parallel$, both of which are independent of $\matV_1^\perp \matX \matR_1^\perp$. We conclude that conditioned on $v^{(1)}$ and $w^{(1)}$, $\matV_1^\perp \matX \matR_1^\perp$ has independent $\mathcal{N}(0,\frac{1}{d})$ entries and $\matV_1\matX\matR_1^\parallel$ is independent of $\matV_1^\perp\matX\matR_1^\perp$ which concludes the base case.

\paragraph{Inductive step:}
Suppose that, conditioned on the first $T-1$ queries $v^{(1)},\dots,v^{(T-1)}$ and observations $w^{(1)},\dots,w^{(T-1)}$, the matrix $\matV_{1:T-1}^\perp \matX \matR_{1:T-1}^\perp$ has independent $\mathcal{N}(0,\frac{1}{d})$ entries and that $\matV_{1:T-1}\matX\matR_{1:T-1}^\parallel$ is independent of $\matV_{1:T-1}^\perp\matX\matR_{1:T-1}^\perp$.

First, we note that
\begin{equation}\label{eq:inductive-step-eq1}
\matV_{1:T} \matX \matR_{1:T-1} = \matV_T\begin{bmatrix}
\matV_{1:T-1}^\parallel \matX \matR_{1:T-1}^\parallel & 0 \\
\matV_{1:T-1}^\perp \matX \matR_{1:T-1}^\parallel & \matV_{1:T-1}^\perp \matX \matR_{1:T-1}^\perp
\end{bmatrix}
\end{equation}
By the inductive hypothesis, conditioned on $v^{(1)},\dots,v^{(T-1)}$ and $w^{(1)},\dots,w^{(T-1)}$, the the matrix $\matV_{1:T-1}^\perp \matX \matR_{1:T-1}^\perp$ has independent $\mathcal{N}(0,\frac{1}{d})$ entries, and it is also independent of $\matV_{1:T-1}\matX\matR_{1:T-1}^\parallel$. Since $\matV_T$ is independent of everything conditioned on $v^{(1)},\dots,v^{(T-1)}$ and $w^{(1)},\dots,w^{(T-1)}$, it also follows that $\matV_T$ is independent of $\matV_{1:T-1}^\perp \matX \matR_{1:T-1}^\perp$. 

Thus, as before, since $\matV_{1:T}$ is orthonormal and since orthogonal components of Gaussian vectors are independent, it follows that, conditioned on $v^{(1)},\dots,v^{(T)}$ and $w^{(1)},\dots,w^{(T-1)}$:
(1) the bottom-right $(d-T+1)\times(d-T+1)$ submatrix of $\matV_{1:T} \matX \matR_{1:T-1}$ has independent $\mathcal{N}(0,\frac{1}{d})$ entries and
(2) this submatrix is independent of the rest of the matrix $\matV_{1:T} \matX \matR_{1:T-1}$.

The matrix $\matR_T$ is measurable with respect to just the first $T$ rows of $\matV_{1:T} \matX \matR_{1:T-1}$, so the above observation implies that $\matR_T$ is independent of the bottom-right $(d-T)\times(d-T+1)$ submatrix of $\matV_{1:T} \matX \matR_{1:T-1}$ conditioned on the previous queries and observations. Therefore, conditioned on $v^{(1)},\dots,v^{(T)}$ and $w^{(1)},\dots,w^{(T-1)}$:
(1) the bottom $(d-T)\times(d-T+1)$ submatrix of $\matV_{1:T} \matX \matR_{1:T}$ has independent $\mathcal{N}(0,\frac{1}{d})$ entries and
(2) this submatrix is independent of the rest of the matrix $\matV_{1:T} \matX \matR_{1:T}$.

Finally, in light of \eqref{eq:decomposition-of-wishart}, the $T^{\mathrm{th}}$ observation is
\begin{equation}
\begin{aligned}
w^{(T)} &= \matV_{1:T}^\top \pars{\matV_{1:T}\matX \matR_{1:T}}\pars{\matV_{1:T}\matX \matR_{1:T}}^\top e_T \\
&= \matV_{1:T}^\top
\bracks{\pars{\matV_{1:T}^\parallel \matX \matR_{1:T}^\parallel}\pars{\matV_{1:T}^\perp \matX \matR_{1:T}^\parallel}^\top}[:,1]
\end{aligned}
\end{equation}
Therefore, $w^{(T)}$ is measurable with respect to:
(1) $\matV_{1:T}$, which is measurable with respect to $v^{(1)},\dots,v^{(T)}$, and
(2) $\matV_{1:T}^\parallel \matX \matR_{1:T}^\parallel$ and $\matV_{1:T}^\perp \matX \matR_{1:T}^\parallel$, i.e.~$\matV_{1:T}\matX\matR_{1:T}^\parallel$, i.e.~the first $T$ columns of $\matV_{1:T}\matX\matR_{1:T}$. Therefore, based on our conclusions above, conditioned on $v^{(1)},\dots,v^{(T)}$ and $w^{(1)},\dots,w^{(T)}$:
(1) the bottom $(d-T)\times(d-T)$ submatrix of $\matV_{1:T} \matX \matR_{1:T}$---i.e.~$\matV_{1:T}^\perp\matX\matR_{1:T}^\perp$---has independent $\mathcal{N}(0,\frac{1}{d})$ entries and
(2) it is independent of the rest of the matrix $\matV_{1:T} \matX \matR_{1:T}$. 

Therefore, by induction, $\Wcorner := (\matV_{1:T}^\perp \matX \matR_{1:T}^\perp)(\matV_{1:T}^\perp \matX \matR_{1:T}^\perp)^\top$ satisfies $(\frac{d}{d-T}) \cdot \Wcorner \sim \Wishart(d-T)$, which completes the proof.

}



\paragraph{Acknowledgements}
MB's research is supported in part by the NSF Alan T. Waterman Award, Grant No. 1933331, a Packard Fellowship in Science and Engineering, and the Simons Collaboration on Algorithms and Geometry. EH is supported in part by NSF grant \#1704860. MS is supported by an Open Philanthropy AI Fellowship, and this work was conducted while visiting Princeton University.
BW is supported by the Google PhD fellowship program, and this work was conducted while an intern at Google AI Princeton. We thank Ramon Van Handel for his helpful discussions regarding random matrix theory.
\clearpage
\bibliographystyle{plainnat}
\bibliography{main}
\clearpage

\clearpage
\appendix

\section{Proof of Proposition~\ref{prop:eig_lin_reduc} \label{sec:lin_redux_proof}}
\newcommand{\uhat}{\widehat{\mathsf{u}}}

The proof of Proposition~\ref{prop:eig_lin_reduc} has two steps. First, we show that if $\Alg$ is $\QAA$ that can solve a linear system to \emph{high precision} (in the Euclidean norm), then $\Alg$ implies the existance of an $\AlgEig$ which can recover the top eigenvector of a matrix up to roughly that precision:
\begin{lemma}[Shift-and-Invert Reduction] \label{lem:shift_and_invert}For parameters $\alpha > 0$ and $\gappar \ge 1$, and recall the set
\begin{align*}
\Mclass(\gappar,\alpha) := \left\{M \in \Sympl^d: \gap(M) \ge \gappar,~ |\lambda_1(M)-1|\le \alpha\, \gappar,~ \lambda_1(M) \in \left[\tfrac{1}{2},2\right]\right\},
\end{align*}
and set $\gappar_{\alpha} := \frac{1}{3+4\alpha}$ and $\condpar_{\alpha} = \frac{1}{\gappar} + (1+\alpha)$. Further, suppose $\Alg$ is a $\QAA$ with query complexity $T$, and satisfies, for a given $\epsilon \in (0,1)$, and for all $A\in \PD$ with $\cond(A) \le \condpar_{\alpha}$ and $b \in \R^d$,
\begin{align*}
 \Pr_{A,b,\Alg}\left[\|\xhat - A^{-1}b\|_2^2 \le \left(\tfrac{\epsilon \gappar_{\alpha}}{5}\right)^2 \|A^{-1}b\|^2_2\right] \ge 1 -\delta
\end{align*}
Then, for any $\tau  \ge 1$ and $d$ for which $\epsilon \le \frac{1}{\tau \sqrt{d}}$, and foran $R(\epsilon,\alpha) = \BigOh{\frac{\log(1/\epsilon)}{\gappar_{\alpha}}}$, there exists an an $\EAA$, $\AlgEig$, which has query complexity $\Quer(\AlgEig) \le \Quer(\Alg) \cdot R(\epsilon,\alpha)$, and satisfies 
\begin{align*}
\quad \Pr_{M,\AlgEig}\left[\vhat^\top M \vhat \ge \lambda_1(M)(1 - \epsilon^2)\right] \ge 1- \delta R - \BigOh{\tau^{-1}} - e^{-\Omega(d)}, \quad \forall M\in \Mclass(\gappar,\alpha) 
\end{align*}
\end{lemma}
This lemma is obtained by the so-called \emph{shift-and-invert} procedure, which approximates $v_1(M)$ by running the power method on the $A^{-1}$, where $A = \gamma I - M$ is a ``shifted'' version of $M$ for an appropriate shift parameter $\gamma$.

Second, we show that if $\Alg$ can solve a linear system to moderate $\BigOh{\frac{1}{\gap}}$-precision in the $\|\cdot\|_A$, it can be bootrsapped to obtain high precision solutions in $\|\cdot\|_2$: 
\begin{lemma}[Bootstrapping Moderate Precision Solves] \label{lem:lin_sys_bootstrap} Fix $\condpar \ge 1$, and suppose $\Alg$  satisfies, for all $A:\cond(A) \le \condpar$, 
\begin{align}
\Pr_{x_0,b,\Alg}\left[\|\xhat - A^{-1}b\|^2_{A} \le \frac{\|x_0-A^{-1}b\|^2\lambda_1(A)}{e\condpar}\right] \ge 1- \tfrac{1}{e}\label{eq:base_guarantee}
\end{align}
Then, any for any $\epsilon,\delta \in (0,1/e)$, there exist a $\QAA$, $\Alg'$ with $\Quer(\Alg') \le \Quer(\Alg) \cdot \BigOh{ Q(\epsilon,\delta)}$ which satisfies
\begin{align*}
\Pr_{A,x_0,b,\Alg}\left[\|\xhat - A^{-1}b\|^2_{2} \le \epsilon\|x_0 - A^{-1}b\|^2_{2}\right] \ge 1 - \delta,
\end{align*}
where $Q(\epsilon,\delta):=  (\log \frac{1}{\epsilon}) \log (\frac{1}{\delta}\log \frac{1}{\epsilon})$,
\end{lemma}

Proposition~\ref{prop:eig_lin_reduc} now follows from combining these two lemmas above

\begin{proof} Let $\tau = \BigOh{d}$ and $d \ge \dmin$, for $\dmin$ to be sufficiently large that the term $\BigOh{\tau} + e^{-\BigOmega{d}} \le 1/2e$. For a parameter $c$, we define the following constants.
\begin{align*}
\epsilon &= \min\{c\,\gappar, 1/\tau\sqrt{d}\} = \Omega(\min \{d^{-3/2},c\,\gappar\})\\
 \epsilon' &:= \left(\frac{\epsilon \gappar_{\alpha}}{5}\right)^2, \quad  \delta := \frac{1}{2e R(\epsilon,\alpha)}.
\end{align*}
Now, suppose that $\Alg_0$ is a $\QAA$ which satisfies~\eqref{eq:base_guarantee} with query compexity $\Quer(\Alg_0) \le T$. Then, by Lemma~\ref{lem:lin_sys_bootstrap} the exists a $\QAA$ $\Alg$ with query complexity $T\cdot Q(\epsilon',\delta)$ satisfying
\begin{align*}
 \Pr_{A,b,\Alg}\left[\|\xhat - A^{-1}b\|_2^2 \le \epsilon' \|A^{-1}b\|^2_2 \right] \ge 1 -\delta,
\end{align*}
Hence, by Lemma~\ref{lem:shift_and_invert}, there exists an $\EAA$ $\AlgEig$ with query complexity $T\cdot Q(\epsilon',\delta)\cdot R(\epsilon,\alpha)$ which satisfies, for all $M \in \Mclass(\gappar,\alpha) $
\begin{align*}
\Pr_{M,\AlgEig}\left[\vhat^\top M \vhat \ge \lambda_1(M) - c\gappar \right] \ge 1- \delta - \BigOh{\tau} + e^{-\BigOmega{d}} \ge 1 - \frac{1}{e}.
\end{align*}
We can increasing the success probability of $\AlgEig$ to $\ge 1 - \delta$ by restarting $\AlgEig$ $L = \BigOh{\log \frac{1}{\delta}}$ times to obtain $\vhat^{(1)},\dots,\vhat^{(L)}$, and returning
\begin{align*}
\overline{\vhat} := \argmax\{v \in \{\vhat^{(j)}\}_{j \in [L]}: v^\top M v\},
\end{align*}
In total, this requires at most $L + LT\cdot Q(\epsilon',\delta)\cdot R(\epsilon,\alpha) = T\BigOh{(\log\frac{1}{\delta}) \cdot  Q(\epsilon',\delta)\cdot R(\epsilon,\alpha)}$ queries. 

We conclude by boudning $Q(\epsilon',\delta)\cdot R(\epsilon,\alpha)$. We have that
\begin{align*}
 R(\epsilon,\alpha) &= \BigOh{\frac{\log(1/\epsilon)}{\gap_{\alpha}}} \le\BigOh{\frac{\log(1/\gappar_{\alpha}\epsilon)}{\gap_{\alpha}}}\\
 Q(\epsilon',\delta) &= \left(\log \frac{1}{\epsilon'}\right) \cdot \left(\log \frac{1}{\delta}(\log \frac{1}{\epsilon'})\right)\\
 &= \BigOh{\left(\log \frac{1}{\gappar_{\alpha}\epsilon}\right) \cdot \left(\log \frac{1}{R(\epsilon.\alpha)} + \log\log \frac{1}{\gappar_{\alpha}\epsilon})\right)}\\
 &= \BigOh{\left(\log \frac{1}{\gappar_{\alpha}\epsilon}\right) \cdot \left(\log \frac{1}{\gappar_{\alpha}} + \log \log \frac{1}{\gappar_{\alpha}\epsilon})\right)}
\end{align*}
Hence
\begin{align*}
 Q(\epsilon',\delta)\cdot R(\epsilon,\alpha) &= \BigOh{\frac{1}{\gappar_{\alpha}} \log^2  \frac{1}{\gappar_{\alpha}\epsilon} \left(\log \frac{1}{\gappar_{\alpha}} + \log \log \frac{1}{\gappar_{\alpha}\epsilon}\right)}\\
&= \BigOh{\frac{\log \frac{1}{\gappar_{\alpha}}}{\gappar_{\alpha}} \log^{2 + \log}  \frac{1}{\gappar_{\alpha}\epsilon}} \\
&= \BigOhPar{\alpha}{\log^{2+\log} \frac{d}{\min\{1,\,c\gappar\}}},
\end{align*}
where we recall the notation $\log^{p+\log}(x) = (\log^p x) \log \log x$. 

\end{proof}

\subsection{Proof of Lemma~\ref{lem:lin_sys_bootstrap}}
	Recall the function $f(x) = \frac{1}{2} x^\top A x - b^\top x$, and note that
	\begin{align*}
	f(x) - f(A^{-1}b) = \|x-A^{-1}b\|_A^2.
	\end{align*}
	Let $q\ge 1$ be a parameter to be selected later, and let $\Alg_q$ denote the algorithm which (a) runs $\Alg$ $q$ times to obtain estimates $\xhat^{(1)},\dots,\xhat^{(q)}$, and (b) makes at most $q$ additional queries to find
	\begin{align}
	\xhat \in \arg\min\left\{f(x): x\in \{\xhat^{(1)},\dots,\xhat^{(q)}\} \right\}.
	\end{align}
	Then, by independence of the internal randomness of $\Alg$, we can ensure 
	\begin{align*}
	\Pr_{x_0,b,\Alg_q}\left[\|\xhat - A^{-1}b\|^2_{A} \le \frac{\|x_0\|^2\lambda_1(A)}{e\condpar}\right] \ge 1- e^{-q}
	\end{align*}
	using at most $Tq + q$ queries. 

	Next, oberve that if $\|\xhat - A^{-1}b\|^2_{A}\le \frac{\|x_0\|^2\lambda_1(A)}{e\condpar} \le \frac{\|x_0\|^2\lambda_1(A)}{e\cond(A)} = \frac{1}{e}\|x_0\|^2\lambda_d(A)$, then $\|\xhat - A^{-1}b\|_2^2 \le \frac{1}{e}\|x_0 - A^{-1}b\|^2_2$. Hence, by repeating $\Alg_q$ $k$-times, each time setting $x_0$ for the $j$-th repetition to coincide with $\xhat$ from the $j-1$st, we find that
	\begin{align*}
	\Pr_{x_0,b,\Alg_q}\left[\|\xhat - A^{-1}b\|^2_{A} \le e^{-k}\frac{\|x_0-A^{-1}b\|^2\lambda_1(A)}{e\condpar}\right] \ge 1- ke^{-q}
	\end{align*}
	Hence, setting $k = \log(\frac{1}{\epsilon})$ and $q = \log \frac{1}{\delta} \log (\frac{1}{\epsilon})$, we obtain the lemma.

\subsection{Proof of Lemma~\ref{lem:shift_and_invert}}
	Let $M \in \Mclass(\alpha,\gappar)$, so that $\lambda_1(M) \in [1/2,2]$, $\gap(M) \ge \gappar$ and $|M - I| \le \alpha \gappar$. We define the associated ``shifted'' matrix $A:= (1+(1+\alpha) \gappar )I - M$. Crucially, $A^{-1}$ and $M$ have the same top eigenvector, and their eigenvalues are related by the correspondence
	\begin{align*}
	\lambda_j(A^{-1}) = \frac{1}{1+2\alpha \gappar - \lambda_j(M)}
	\end{align*}
	We can therefore compute the eigengap of $A$ via
	\begin{align*}
	\frac{ \lambda_2(A^{-1})}{\lambda_1(A^{-1})} &= \frac{1 + (1+\alpha)\gappar - \lambda_1(M)}{1 + (1+\alpha) \gappar - \lambda_2(M)} \quad = \frac{1 + (1+\alpha) \gappar - \lambda_1(M)}{1 + (1+\alpha) \gappar - (1-\gappar) \lambda_1(M)}\\\
	&= \frac{1}{1 + \frac{\gappar \lambda_1(M)}{1 + (1+\alpha) \gappar - \lambda_1(M)}} \overset{(i)}{\le} \frac{1}{1 + \frac{\gappar \lambda_1(M)}{(1+2\alpha) \gappar}} \overset{(ii)}{\le} \frac{1}{1 + \frac{1}{2(1+2\alpha)}} 
	\end{align*}
	where $(i)$ uses $|\lambda_1(M) - 1| \le  \alpha \gappar$, and $(ii)$ uses $\lambda_1(M) \ge 1/2$. Hence, 
	\begin{align*}
	\gap(A^{-1}) \ge \frac{\frac{1}{2(1+2\alpha)}}{1 + \frac{1}{2(1+2\alpha)}} =\frac{1}{1+2(1+2\alpha)} = \frac{1}{3+4\alpha} := \gappar_{\alpha}
	\end{align*}
	In other words, the eigengap of $A^{-1}$ depends on the parameter $\alpha$, but \emph{not} on the eigengap of $M$. Hence we can effectively run the power method on $A^{-1}$ to compute the top eigenvector of $M$. Of course, we cannot query $A^{-1}$, but we can approximate a query $A^{-1}v$ by using a $\QAA$. To facillitate this reduction, we observe that $\cond(A) = \BigOh{\gappar^{-1}}$:
	\begin{claim} $\cond(A) \le \condpar_{\alpha} := \frac{1}{\gappar} + (1+\alpha)$. 
	\end{claim}
	\begin{proof} Since $M \preceq 0$, $\lambda_1(A) \le 1+(1+\alpha)\gappar$. Moreover, since $|\lambda_1(M) -1| \le \alpha \gappar$, $\lambda_{\min}(A) =1+2\alpha\gappar - \lambda_1(M) \ge \gappar$.
	\end{proof}

We our now ready to present the reduction. Let $\Alg$ be $\QAA$ satisfying the condition $\Pr_{A,b,\Alg}\left[\|\xhat - A^{-1}b\|_2^2 \le \left(\tfrac{\epsilon \gappar_{\alpha}}{5}\right)^2 \|A^{-1}b\|^2_2\right] \ge 1 -\delta$ for all $A:\cond(A) \ge \condpar_{\alpha} $ and $ b\in \R^d$,

For a round number $R \ge 1$ to be selected later, precision $\epsilon$, and failure probability $\delta$, we define a procedure $\AlgEig$ in Algorithm~\ref{alg:noisy_power}, which uses $\Alg$ as a primitive to run an approximate power method on $A^{-1}$, up to the errors:
\begin{align*}
\Delta_r := \|\xhat^{(r)} - A^{-1}u_r\|_2.
\end{align*}
\begin{algorithm}[h!]
\textbf{Input: } Confidence Parameter $\epsilon$, accuracy parameter $\delta$\\
 \textbf{Draw} $u_0 \unifsim \sphered$\\
 \For{rounds $r=1,2,\dots,R$}
 {call $\Alg$ to obtain $\xhat^{(r)}$  such that
	\begin{align}
	\Pr\left[\|\xhat^{(r)} - A^{-1}u_{r-1}\|^2_2 \le \left(\frac{\epsilon \gappar_{\alpha}}{5}\right)^2 \|A^{-1}u_{r-1}\|^2_2\right] \ge 1 - \delta \label{eq:alg_guarantee}
	\end{align}
	Set $u_r := \xhat^{(r)}/\|\xhat^{(r)}\|_2$}

	\textbf{Return} $\vhat = u_r$. 
	\caption{$\AlgEig$ (Approximate Power Method via $\Alg$) \label{alg:noisy_power}}
\end{algorithm}

This ``noisy'' power method admits a black-box analysis due to \cite{hardt2014noisy}:
\newcommand{\epsbar}{\overline{\epsilon}}
\begin{lemma}[Corollary 1.1 in \cite{hardt2014noisy}, $k = p = 1$, specialized to Algorithm~\ref{alg:noisy_power}]\label{lem:noisy_power} Fix a parameter $\tau > 1$, and an $\epsilon \le \frac{1}{\tau \sqrt{d}}$. Then, if 
\begin{align*}
\Delta_r \le  \left(\frac{\lambda_1(A^{-1}) - \lambda_2(A^{-1})}{5}\right) \epsilon,
\end{align*}
then for an $R = \BigOh{\frac{\log(d\tau/\epsilon)}{\gap(A^{-1})}}$,  $\sqrt{1 - \langle u_{R}, v_1(M)\rangle^2}  \le \epsilon$ with probability $1 - \BigOh{\tau^{-1}} - e^{-\BigOmega d}$ over the draw of $x_0$, where $c$ is a universal constant. 
\end{lemma} 
We first interpret the bound $\sqrt{1 - \langle u_{R}, v_1(M)\rangle^2}$ in terms of the subotimality $\lambda_1(M) - u_R^\top M u_R$. Since $0 \preceq M \preceq 2I$, we have that if the conclusion of Lemma~\ref{lem:noisy_power} is satisfied, 
\begin{align*}
u_R^\top M u_R &= \lambda_1(M) (u_R^\top v_1(M))^2 + \sum_{i=2}^d \lambda_i(M)\cdot (u_R^\top v_i(M))^2\\
&\ge \lambda_1(M) (u_R^\top v_1(M))^2 = \lambda_1(M) (1 + (1-(u_R^\top v_1(M))^2 )) = \lambda_1(M)(1 - \epsilon^2).
\end{align*}
We can now conclude the proof by verifying
\begin{align*}
&\Pr[u_R^\top M u_R \ge \lambda_1(M)-  2\epsilon^2 ] + \BigOh{\tau} + e^{-\BigOmega{d}} \\
&\ge \Pr\left[\forall r \in[R] \Delta_r \le  \left(\tfrac{\lambda_1(A^{-1}) - \lambda_2(A^{-1})}{5}\right) \epsilon \right] \\
&= \Pr\left[\forall r \in[R] \Delta_r \le  \lambda_{1}(A^{-1})\left(\tfrac{\gap(A^{-1})}{5}\right) \epsilon \right]\\
&\overset{(i)}{\ge} \Pr\left[\forall r \in[R] \Delta_r \le  \|A^{-1}u_{r-1}\|_2\left(\frac{\gappar_{\alpha}}{5}\right) \epsilon \right] \overset{(ii)}{\ge} R\delta,
\end{align*}
where $(i)$ follows from the bound $\gap(A^{-1}) \le \gappar_{\alpha}$ and $\|A^{-1}u_{r-1}\|_2 \le \|u_{r-1}\|_2 \lambda_1(A^{-1}) = \lambda_1(A^{-1})$, and $(ii)$ by a union bound over the event in~\eqref{eq:alg_guarantee}, with $R = \BigOh{\frac{\log(d\tau/\epsilon)}{\gap(A^{-1})}} = \BigOh{\frac{\log(1/\epsilon)}{\gappar_{\alpha}}} $, where we recall $\epsilon \le \frac{1}{\tau\sqrt{d}}$.

\section{Proof of Theorem~\ref{thm:quad_lb} for Arbitrary Condition Number \label{sec:arbitrary_condition}}
In this section, we given proof of Theorem~\ref{thm:quad_lb} for general condition number. Using Proposition~\ref{prop:eig_lin_reduc} directly for matrices with larger gap incurs a dimension on $\log$ of the ambient dimension.

\newcommand{\calV}{\mathcal{V}}
To sharpen this, we state a slightly refined reduction. For this to go through, define, for a subspace $\calV \subset \R^{d}$, let 
\begin{align*}
\Mclass(\gappar,\alpha,\calV) := \{M \in \Mclass(\gappar,\alpha): v_1(M) \in \calV \}
\end{align*}
to denote the restriction of $\Mclass(\gappar,\alpha)$ to matrices whose top eigenvector is \emph{know} to lie in a subspace $\calV$. For this class, we can improve Proposition~\ref{prop:eig_lin_reduc}
as follows: 

\begin{proposition}[Eigenvector-to-Linear-System Reduction, Known Subspace]\label{prop:eig_lin_reduc_subspace} Fix a $\gappar \in (0,1)$, $\alpha > 0$, and suppose that $\Alg$ be a $\QAA$ which which satisfies~\eqref{eq:to_contradict} with $\condpar := 1+\alpha + \frac{1}{\gappar}$ for all $A \in \PD$ with $\cond(A) \le \condpar$. Then, for any $\delta \in (0,1/e)$,  there exists an $\EAA$, $\AlgEig$, which satisfies
\begin{align*}
 \Pr_{\AlgEig,M}\left[\vhat^\top M \vhat \ge (1 - c\gappar)\lambda_1(M) \right] \ge 1 - \delta, \quad \forall M \in \Mclass\left(\gappar,\alpha,\calV\right) 
\end{align*}
with query complexity at most 
\begin{align*}
\Quer(\AlgEig) \le \Quer(\Alg) \cdot \BigOhPar{\alpha}{(\log \frac{1}{\delta}) \cdot \log^{2+\log} \frac{\dim(\calV)}{\min\{c\gappar,\, 1\}}},
\end{align*} where $\BigOhPar{\alpha}{\cdot}$ hides multiplicative and additive constants depending on $\alpha$.
\end{proposition}
\begin{proof}The proof is nearly identical to the proof of Theorem~\ref{prop:eig_lin_reduc}. The only difference is that, by initializing $u_0$ to be uniform on $\sphered \cap \calV$, the guarantee of the noisy power method (Lemma~\ref{lem:noisy_power}) can be improved to depend on $\dim(\calV)$ instead of the ambient dimension $d$. 
\end{proof}
The proof of Theorem~\ref{thm:quad_lb} for general $\condpar$ is as now as follows: Fix $s \in [d_0(1/2) \vee \dmin,d]$. Let $\beta = \frac{1}{2}$, $\matM \sim \calD(s,s,\beta)$, and define the embedded matrix $\overline{\matM} = \begin{bmatrix} \matM & \mathbf{0} \\ \mathbf{0} & \mathbf{0} \end{bmatrix} \in \R^{d \times d}$. Then, $v_1(\overline{\matM})$ lies in the $s$-dimensional subspace $\calV$ corresponding to the first $s$ entries, it is easy to check that $\overline{\matM} \in \Mclass(\gappar,\alpha,\calV)$, where  $\gappar := \frac{\cgap(\beta)}{s^2}$ and $\alpha := \frac{\ctopeig(\beta)}{\cgap(\beta)}$ analogously to the $s = d$ case of Theorem~\ref{thm:quad_lb}. 

Retracing the steps, and replacing the dependence of $d$ with $s$, we find if $\Alg$ satisfies the guarantee of~\eqref{eq:to_contradict} for all $\cond(A) \le \condpar := 1 + \alpha + \frac{1}{\gappar} = \Theta(s^2)$,
\begin{align*}
\Quer(\Alg) \ge \frac{\cquery}{\sqrt{\gappar}} \cdot  \BigOmega{ (\log^{2+\log} s)^{-1}} = \BigOmega{s (\log^{2+\log} s)^{-1}}.
\end{align*}

\end{document}